\documentclass{article} % For LaTeX2e
\usepackage{iclr2025_conference,times}
\usepackage{amsmath,amsfonts,bm}

\def\eqref#1{equation~\ref{#1}}
\def\1{\bm{1}}

\DeclareMathAlphabet{\mathsfit}{\encodingdefault}{\sfdefault}{m}{sl}
\SetMathAlphabet{\mathsfit}{bold}{\encodingdefault}{\sfdefault}{bx}{n}

\newcommand{\softmax}{\mathrm{softmax}}
\newcommand{\sigmoid}{\sigma}

\usepackage{hyperref}
\usepackage{url}
\usepackage{graphicx}
\usepackage{algorithm}
\usepackage{algpseudocode}
\usepackage{caption}
\usepackage{subcaption}
\usepackage{wrapfig}

\usepackage{amsthm}
\usepackage{amssymb}
\usepackage{subcaption}
\usepackage{graphicx}

\newtheorem{theorem}{Theorem}
\newtheorem{lemma}[theorem]{Lemma}
\newtheorem{proposition}[theorem]{Proposition}

\theoremstyle{definition}
\newtheorem{remark}[theorem]{Remark}

\theoremstyle{remark}

\title{VerifierQ: Enhancing LLM Test Time Compute with Q-Learning-based Verifiers}

\author{
Jianing Qi$^{1}$,~~Hao Tang$^{1,2}$,~~Zhigang Zhu$^{1,3}$\\ 
$^1$CUNY Graduate Center, $^2$Borough of Manhattan Community College, $^3$The City College of New York\\ 
\texttt{jqi@gradcenter.cuny.edu}\\
\texttt{htang@bmcc.cuny.edu}\\
\texttt{zzhu@ccny.cuny.edu}
}

\iclrfinalcopy % Uncomment for camera-ready version, but NOT for submission.
\begin{document}

\maketitle

\begin{abstract}
   Recent advancements in test time compute, particularly through the use of verifier models, have significantly enhanced the reasoning capabilities of Large Language Models (LLMs). This generator-verifier approach closely resembles the actor-critic framework in reinforcement learning (RL). However, current verifier models in LLMs often rely on supervised fine-tuning without temporal difference learning such as Q-learning. This paper introduces VerifierQ, a novel approach that integrates Offline Q-learning into LLM verifier models. We address three key challenges in applying Q-learning to LLMs: (1) handling utterance-level Markov Decision Processes (MDPs), (2) managing large action spaces, and (3) mitigating overestimation bias. VerifierQ introduces a modified Bellman update for bounded Q-values, incorporates Implicit Q-learning (IQL) for efficient action space management, and integrates a novel Conservative Q-learning (CQL) formulation for balanced Q-value estimation. Our method enables parallel Q-value computation and improving training efficiency. While recent work has explored RL techniques like MCTS for generators, VerifierQ is among the first to investigate the verifier (critic) aspect in LLMs through Q-learning. This integration of RL principles into verifier models complements existing advancements in generator techniques, potentially enabling more robust and adaptive reasoning in LLMs. Experimental results on mathematical reasoning tasks demonstrate VerifierQ's superior performance compared to traditional supervised fine-tuning approaches, with improvements in efficiency, accuracy and robustness.  By enhancing the synergy between generation and evaluation capabilities, VerifierQ contributes to the ongoing evolution of AI systems in addressing complex cognitive tasks across various domains.
\end{abstract}

%%%%%%%%%%%%%%Introduction%%%%%%%%%%%%%%
\section{Introduction}
Large Language Models (LLMs) offer a promising approach to multi-step reasoning tasks through language. However, despite their prowess in generating coherent text, LLMs face significant challenges in sustained, multi-step logical reasoning due to their underlying architecture and propensity for hallucinations \citep{lightman2023lets}. Overcoming these challenges is critical for enabling the next level of agent capabilities.

One of the most important recent developments in addressing these limitations is test time compute \citep{snell2024scalingllmtesttimecompute,cobbe2021training}. As demonstrated by \citet{openai2024learning}, test time compute represents a new paradigm in the scaling laws of LLMs. This approach essentially involves using a verifier model to evaluate and select the best solutions generated by an LLM, allowing for more extensive processing and deliberation during inference. By leveraging additional computational resources at test time, LLMs can potentially perform more complex reasoning tasks with improved accuracy and reduced hallucinations.

The concept of a verifier aligns closely with recent research on multi-step reasoning, which typically employs two main components: a \textbf{generator} and a \textbf{verifier} \citep{lightman2023lets, uesato2022solving, cobbe2021training}. The generator produces potential solutions, while the verifier evaluates their correctness. This setup is analogous to the actor-critic framework in Reinforcement Learning (RL) \citep{Konda1999ActorCriticA}. However, unlike RL critics that use temporal-difference (TD) updates for long-term credit assignment, current verifiers in multi-step reasoning are often trained using supervised fine-tuning (SFT). This limitation might hinder the verifier's ability to guide the generator towards better long-term outcomes, particularly in complex reasoning tasks.

To address these challenges, we propose leveraging Reinforcement Learning techniques, particularly Offline Q-learning, to enhance verifier performance in long-horizon tasks. This approach draws inspiration from successful RL systems like AlphaGo \citet{alphago}, which achieve superhuman performance by combining learning and search techniques. Recent research has focused on improving generators using methods like Monte Carlo Tree Search (MCTS) \citep{chen2024alphamath,wang2024mathshepherd, wang2024qimprovingmultistepreasoning, zhang2024accessinggpt4levelmathematical}. However less attention has been given to applying RL techniques to verifiers.

We introduce VerifierQ, an Offline Q-learning approach that integrates classical reinforcement learning techniques with LLMs. The core research question guiding this work is: Can Offline Q-learning improve the verifier's ability to handle multi-step reasoning tasks, and if so, how can we overcome the obstacles that currently limit its application to LLM value networks? 

Our work makes several key contributions:(1) We propose a flexible architecture for applying Q-learning on utterance-level in language models, and we resolve large action space problem on utterance level. (2) We present an innovative formulation of Conservative Q-learning tailored for these large action spaces, mitigating overestimation issues in offline Q-learning. (3) Our approach bridges the gap between classic critic models in reinforcement learning and verifier models in language tasks, opening new avenues for improving test-time compute in large language models. 

By combining the strengths of test time compute and RL-based verifiers, we seek combine RL's success in reasoning and planning into LLM. Our goal is to explore the potential of reinforcement learning methods to increase the verifier's capacity for efficient long-term reasoning, ultimately leading to better performance in these challenging domains.

%%%%%%%%%%%%%%Backgound%%%%%%%%%%%%%%
\section{Background}
\label{sec:background}

In reinforcement learning (RL), tasks are often modeled as Markov Decision Processes (MDPs). A MDP consists of states $s \in S$, actions $a \in A$, a reward function $r(s, a)$, a transition function $P(s' | s, a)$  from state $s$ to state $s'$ with action $a$, and a discount factor $\gamma$. The goal is to find an optimal policy $\pi^*$ that maximizes the expected cumulative reward:  
$\pi^* = \arg\max_{\pi} \mathbb{E} \left[ \sum_{t=0}^{\infty} \gamma^t r_t \right]$  

Q-learning estimates the optimal states action value $Q$ by iteratively updating expected cumulative rewards, while the $V$ function is similar  to $Q$ but just estimates from states without need of actions. Q-learning is a model-free RL algorithm commonly used to solve MDPs by minimizing the temporal difference (TD) error:  
\begin{equation}
   L_{TD}(\theta) = \mathbb{E} \left[ \left( r + \gamma \max_{a'} Q(s', a'; \theta) - Q(s, a; \theta) \right)^2 \right].
   \label{eq:td}
\end{equation}
In the context of large language models (LLMs) using a generator-verifier framework, the generator's output can be treated as a sequence of actions. Given a problem statement $p$, the generator produces a solution as a sequence of action steps $[a_1, a_2, \dots, a_n]$, where each step $a_i$ is an action. The state at step $i$ is the concatenation of the problem statement and all tokens generated up to that point:  $s_i = [p, a_1, a_2, \dots, a_i]$.
Rewards are given at each step, with a reward of 1 for a correct step and 0 for an incorrect one. We can see the illustrated example in Figure~\ref{fig:RLLLM}, $+$ is correct and $-$ is incorrect. 
In classical RL, critic model is trained to estimate the Q value \citep{Konda1999ActorCriticA}. In test time compute, the verifier model is trained to estimate the Q value \citep{snell2024scalingllmtesttimecompute}.
Given a problem statement, the generator produces a sequence of actions as steps, and the verifier evaluates each step to determine its correctness.

\begin{figure}[t]
   \centering
   \includegraphics[width=1.0\linewidth]{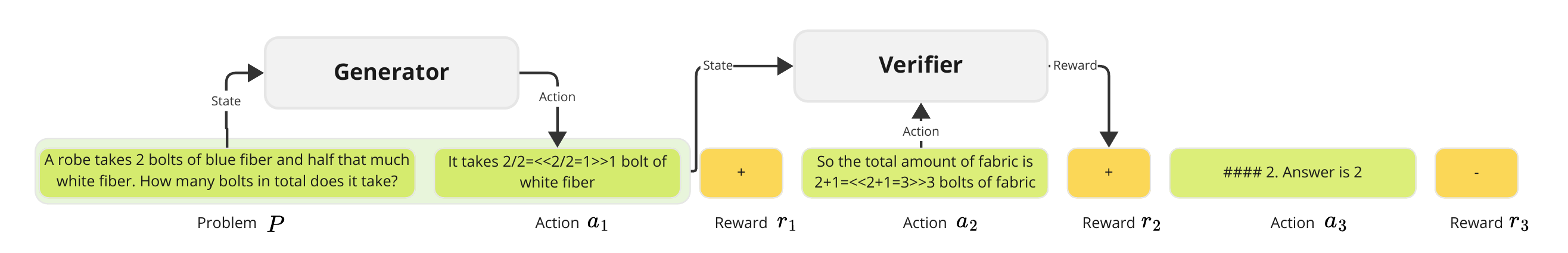}
    \caption{Illustration of State, Action (green), and Reward (orange) in a Math Problem. $+$ denotes correct (1) and $-$ denotes incorrect (0). A state generator produces an action (next solution step). The verifier assesses the existing state and action and outputs a probability of correctness.}
   \label{fig:RLLLM}
\end{figure}

Offline Q-learning, which uses a fixed dataset to estimate Q-values, has the advantage of being more efficient to train compared to online methods. However, it comes with the risk of overfitting to the training data, as it lacks the exploration of new actions typically seen in online Q-learning, and it causes overestimation problem. Recent methods involves Conservative Q Learning and Implicit Q Learning to mitigate the issue of overestimation, and this work tries to combine both approaches to overcome the overestimation problem \citep{kumar2020conservativeqlearningofflinereinforcement, kostrikov2021offline}

%%%%%%%%%%%%%%Related Work%%%%%%%%%%%%%%
\section{Related Work}
Multi-step reasoning, particularly for mathematical tasks, commonly employs a generator-verifier framework to enhance LLM performance. Process Reward Modeling (PRM), which assigns rewards to each step, has been shown to outperform Object Reward Modeling (ORM), where rewards are only given for the final output \citep{lightman2023lets, uesato2022solving, cobbe2021training}. However, PRM requires extensive step-level labeling. Recent works mitigate this with Monte Carlo Tree Search (MCTS) for automatic labeling, improving efficiency over manual methods \citep{chen2024alphamath, wang2024mathshepherd, wang2024qimprovingmultistepreasoning}. Nonetheless, verifiers in these approaches are trained via supervised fine-tuning (SFT), and compared to classic RL, it is similar to imitation learning.

Recent studies emphasize the role of verifiers in improving test-time compute. \citet{cobbe2021training} showed that an effective verifier can yield performance gains equivalent to increasing generator size by 30x. Similarly, \citet{snell2024scalingllmtesttimecompute} found that optimal verifiers at test time can outperform generators 14x larger. However, these methods still rely on SFT for training verifier, limiting their effectiveness in complex, long-horizon reasoning. A more effective learning method can improve the verifier model to achieve better performance and in turn scale more efficiently.

Q-learning has seen limited use in LLMs, mainly for preference-based tasks. ILQL is the first to show implicit Q-learning can be applied to LLMs for multi-turn dialogue tasks, but the focus is on token-level actions \citep{snell2023offline}. To resolve the challenges in training long-horizon reasoning due to the granularity of token-level actions, ArCHer proposed a utterance level value function, but encoder style of value function makes estimation can compute step by step and is less efficient\citep{zhou2024archer}. Both works are still limited by their step by step computations. 

%%%Need to contrast with our work here%%%
In contrast, our work focuses on utterance-level actions for verifier, and multi reward estimation with one forward pass significantly improves training efficiency. We also integrate Implicit Q-learning (IQL) and Conservative Q-learning (CQL) to better manage large action spaces and enhance performance, offering a more scalable solution for multi-step reasoning tasks.

%%%%%%%%%%%%%%Problem Statement%%%%%%%%%%%%%%
\section{Problem Statement}
The central question we need to ask is: Can Offline Q-learning improve the verifier’s ability to handle multi-step reasoning tasks? If so, how can we overcome the obstacles that currently limit its application to LLM value networks? This section is divided into two problems, offline Q-learning vs Imitation Learning and challenges in applying Offline Q-learning to LLMs.

\textbf{Offline Q-learning vs Imitation Learning:}

The characteristics of MCTS rollout data in math problems is that it is noisy. There might be many steps that are not optimal and incorrect for solving the problem. However, stitching the optimal steps together might be able lead to a better solution. Those are cases that Offline Q-learning can handle better than Imitation Learning, and one of the main conclusion from \citet{kumar2022preferofflinereinforcementlearning} is that given the same noisy expert data, Offline Q-learning can outperform Imitation Learning on long horizon tasks. Intuitively, offline RL method should learn to stitch the suboptimal paths in the noisy data to obtain a better path, and wrong answers can help offline RL what is wrong for the future. It could lead to the better performance of the verifier model even with the same amount of rollout from generator.

\textbf{Challenges in applying Offline Q-learning to LLMs:}

There are several challenges in applying Offline Q-learning to LLMs. 
The first challenge is utterance level RL. Existing works uses token level actions \citep{snell2023offline}. At the token level, the action space has a smaller cardinality and is equivalent to the vocabulary size $V$, making it feasible to compute max Q-values and estimates. However, this level of granularity is too fine for long-horizon tasks \citep{zhou2024archer}. On the other hand, using the utterance level allows better handling of long-term horizons, but since each utterance may contain multiple tokens, the action space grows exponentially large, making it computationally intractable \citep{wang2024qimprovingmultistepreasoning}. Given the utterance with number of tokens of length $n$, we will have $V^n$ actions, and a typlical sentence might contain 20 tokens. This is just for one utterance, but a solution contains many utterances. This creates a tension between choosing a level of granularity that is manageable but also effective for long-term planning. We need to find a practical method to make the verifier model to learn on the utterance level efficiently.

The second challenge with most methods is that they rely on an actor to sample actions, because it is hard to estimate the maximum Q-value when the action space is large. Since datasets typically consist of rollouts with various actions, true offline learning becomes difficult. It is not practical to have each step to have a number of samples while maintaining previous steps to be the same. Given a sentence with $m$ steps, if we sample $l$ actions for each step, then we will have $l^m$ different answers for one problem if we want to apply offline learning. Additionally, for Q-value estimation, finding the max Q is problematic due to the large action space, as most of methods require an MCTS actor to sample the maximum value of each step \citep{chen2024alphamath,wang2024qimprovingmultistepreasoning}. This approach doesn’t effectively utilize offline datasets, complicating training. Approximating $\max Q$ needs sampling, and online sampling is not efficient for training. While it is easy to roll out a complete solution, it seems to be difficult to utilize the data for training the verifier model. We need to find a way to make the verifier model to learn efficiently with offline datasets.

The third challenge is overestimation. The overestimation problem in Q-learning is well-known, but it’s particularly severe in language models. As noted in \citet{verma2022chaichatbotaitaskoriented,zhou2024archer}, this issue is amplified in language tasks because the Q-function is trained only on the responses in a fixed dataset, making it unlikely to predict accurate values for arbitrary strings in an utterance. In our preliminary experiments, we’ve observed that changing just one critical token to incorrect utterances can receive a higher value than the correct ones. This overestimation becomes more pronounced at the utterance level, where the complexity and potential for incorrect value assignments are greater.

%%%%%%%%%%%%%%Verifier Q Method%%%%%%%%%%%%%%
\section{Verifier with Q-Learning (VerifierQ)}
We introduce VerifierQ, a novel approach enhancing verifier models using Offline Q-learning for Large Language Models (LLMs). Our method addresses key challenges by modifying the Q-learning algorithm and integrating it into language modeling tasks.

\begin{figure}[t]
   \centering
   \includegraphics[width=1.0\linewidth]{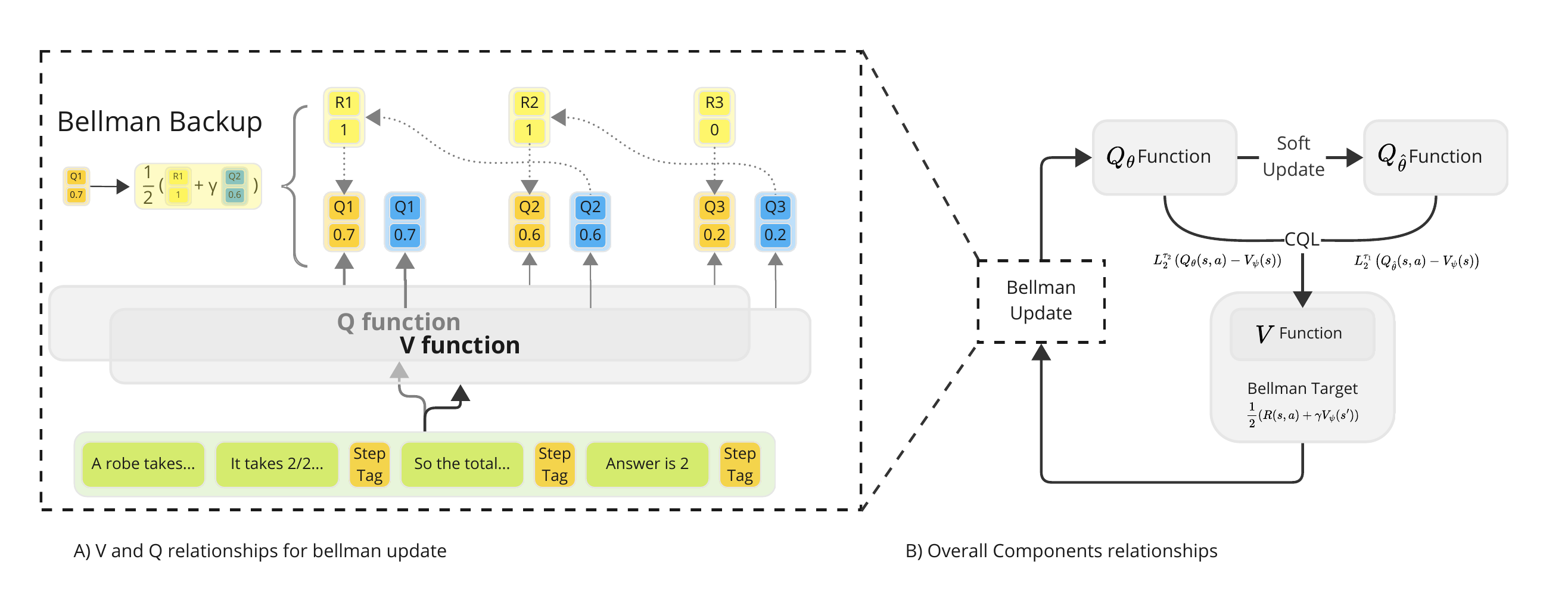}
   \caption{Illustration of the VerifierQ architecture and modified Bellman update. Left: Bellman update, where $Q_{\theta}$ is updated via the TD target with $V$. Right: Relationships among $Q_{\theta}$, $Q_{\hat{\theta}}$, and $V_{\psi}$. $V_{\psi}$ is updated through CQL, $Q_{\theta}$ through the Bellman equation, and $Q_{\hat{\theta}}$ via soft update.}
   \label{fig:verifierq}
\end{figure}

\subsection{Architecture of VerifierQ}
\textbf{Addressing Utterance-Level MDP:} To apply Offline Q-learning to LLMs at the utterance level, we propose a flexible architecture that integrates with language modeling tasks (Figure~\ref{fig:verifierq}). Following \citet{wang2024mathshepherd} and \citet{lightman2023lets}, we utilize two tokens $+$ and $-$ to represent correct and incorrect states, with a tag token indicating estimation. The probability of the correct token out of two tokens can be interpreted as Q-values ranging from 0 to 1, aligning with the reward structure in the MCTS-generated dataset from \citep{wang2024mathshepherd}. More details are provided in the supplementary material Appendix~\ref{app:architecture}.

To address the bounded nature of outputs (0 to 1) compared to traditional Q-values, we modify the Q-learning algorithm to operate within these constraints. We propose using the mean of the current Q-value and the traditional Bellman update as the target value instead of $r + \gamma \max_{a'} Q(s', a'; \theta)$ from Equation~\ref{eq:td}:
\begin{equation}
   Q^*(s, a) = \frac{1}{2} (R(s,a) + \gamma \max_{a'} Q^*(s', a'))
   \label{eq:bellman}
\end{equation}
where $Q^*(s, a)$ is the optimal Q-value for state $s$ and action $a$, $R(s,a)$ is the immediate reward, $\gamma$ is the discount factor, and $\max_{a'} Q^*(s', a')$ is the maximum Q-value for the next state $s'$. More details are provided in the supplementary material (Theorem~\ref{thm:modified_bellman_convergence} of Appendix~\ref{app:architecture}).

For each problem and solution sequence $[p, a_1, a_2, \dots, a_n]$, we insert the tag token at the end of each step $a_i$ and predict the output for the subsequent $+$ or $-$ tokens as shown in Figure~\ref{fig:RLLLM}. The mean of the reward and next estimate serves as the target value for the Bellman update. 

As shown in Figure~\ref{fig:verifierq}, the verifier model estimates multiple Q-values for each step in the solution sequence, enabling efficient parallel computation of Q-values for multiple steps in a single forward pass. This architecture change enables VerifierQ to efficiently learn Q-values at the utterance level while maintaining compatibility with existing language modeling frameworks.

\subsection{Algorithm of VerifierQ}

\textbf{Addressing Large Action Spaces:} Traditional action selection in MDPs typically requires finding the maximum Q-value explicitly over all possible actions. In utterance-level MDPs, this leads to exponentially large action spaces of $V^n$, where $V$ is the vocabulary size and $n$ is the length of tokens in one utterance. To address this challenge, we employ Implicit Q-learning (IQL) \citep{kostrikov2021offline}. 

IQL approximates Q-values through regression on existing actions, mitigating the need for explicit maximum Q-value sampling and enabling efficient handling of limited per-step data. Instead of iteratively finding the maximum Q for every single action in $V^n$, it can regress the action based on the dataset and find the approximation through expectile. IQL can still approximate the maximum Q-value $\max_{a \in \mathcal{A}} Q(s, a)$ without explicitly evaluating all actions by fitting $Q(s, a)$ to the expectiles of the target values given limited data. It improves sample efficiency by eliminating the need for an online algorithm to sample from. This regression-based approach makes IQL particularly well-suited for utterance-level MDPs.

We follow \citet{snell2023offline} for the IQL framework to our setting, using the expectile of the Q-value to approximate the value function $V$:
\begin{equation}
   \mathcal{L}{V}(\psi) = \mathbb{E}_{(s,a) \sim \mathcal{D}} \left[ L_{2}^{\tau} \left( Q_{\theta}(s, a) - V_{\psi}(s) \right) \right]
   \label{eq:iql}
\end{equation}
where $L_{2}^{\tau}(u) = |\tau - \1\mathrm{(u < 0)}| u^2$, $\tau \in (0, 1)$ is the quantile level, $\mathcal{D}$ is the offline dataset, $Q_{\theta}$ is the learned Q-function, and $V_{\psi}$ is the approximated value function.

This formulation allows for efficient Q-value estimation without explicit maximization over all possible actions. Theoretically, as $\tau$ approaches 1, we have $\lim_{\tau \to 1} V_{\psi}(s) = \max_{a} Q^*_{\theta}(s, a)$ \citet{kostrikov2021offline}, ensuring that our IQL-based approach can asymptotically recover the optimal value function, even with large action spaces, given sufficient coverage in the offline dataset. 

Our approach leverages regression to solve the large action space problem. By regressing the reward for each utterances, we can find the approximation of the maximum Q-value and minimum Q-value without needing to sample the action or iterate through all the combinations of tokens. We focus on this regression aspect, not just the data support aspect. The approximation of minimum through $\tau$ value is shown in Theorem~\ref{thm:iql_optimality} of Appendix~\ref{app:architecture} .

\textbf{Addressing Overestimation:} Q-learning often suffers from overestimation bias, particularly severe in language models with large action spaces and limited offline datasets. To mitigate this, we incorporate Conservative Q-learning (CQL) \citet{kumar2020conservativeqlearningofflinereinforcement} into our framework. CQL penalizes Q-values exceeding the target value, making the Q-function more conservative.

We add the following CQL term to the Bellman equation:
\begin{equation}
   \arg \min_{Q} \alpha (\mathbb{E}_{s \sim D, a \sim \mu} \left[ Q(s, a) \right] - \mathbb{E}_{s \sim D, a \sim \hat{\pi}_{\beta}} \left[ Q(s,a) \right])
   \label{eq:cql}
\end{equation}
Where $\mu$ is the target policy distribution and $\hat{\pi}_{\beta}$ is the data distribution. This term minimizes the maximum Q-value under the target policy distribution while maximizing it under the data distribution, providing a tighter bound.

Unlike token-level approaches, we leverage IQL to approximate Q-values in the large action space, mitigating the need to sample a set number of actions for each state and allowing more efficient Q-value estimation for longer sequences.

We propose a novel formulation that directly approximates both the lower bound Q-function and the upper bound of the data distribution using IQL:

\begin{equation}
L_{CQL}(\psi) = \alpha (\mathbb{E}_{s \sim D, a \sim \mu} \left[ L_{2}^{\tau_1} \left( Q_{\hat{\theta}}(s, a) - V_{\psi}(s) \right)\right] - \mathbb{E}_{s \sim D, a \sim \hat{\pi}_{\beta}} \left[ L_{2}^{\tau_2} \left( Q_{\theta}(s, a) - V_{\psi}(s) \right)\right])
\label{eq:cqliql}
\end{equation}

Here, $\tau_1$ is chosen to be close to 0 and $\tau_2$ close to 1, allowing for a more optimistic Q-value estimation within the CQL framework. This approach maintains CQL's conservatism while allowing for adaptable control through the adjustment of $\tau_1$ and $\tau_2$. The lower bound of the target policy pushes the Q-value down less aggressively, while the upper bound of the data distribution elevates it more, resulting in a more adjustable conservatism under the CQL term. For more details on the explanations of the CQL term, see Appendix~\ref{prop:modified_cql_bounds}.

\begin{figure}[t]
   \centering
   \includegraphics[width=\linewidth]{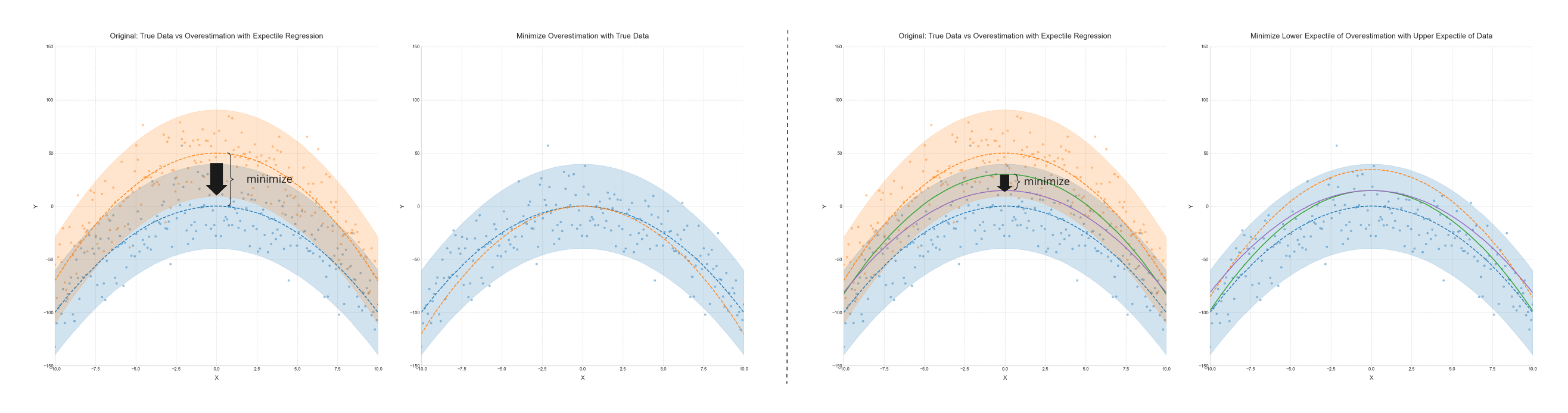}
   \caption{Illustration of our approach. Left: Orange line represents the overestimated Q-value $Q_{\hat{\theta}}$. Blue line indicates the data distribution $Q_{\theta}$. Minimizing the overestimation term brings the orange line down to the mean of data distribution. Right: Green line shows the lower expectile of the overestimated Q-value and purple line shows the upper expectile of the data Q-value. Minimizing those two can make orange line approaches the maximum Q-value under the data distribution.}
   \label{fig:cql-demo}
\end{figure}

Figure~\ref{fig:cql-demo} illustrates the intuition. In the original CQL term, an overestimated Q-value would be pushed down to the data distribution. In our formulation, the lower bound of the Q-value is pushed down less aggressively, and the upper bound is elevated more, resulting in a more optimistic Q-value that approaches the maximum Q-value under the data distribution more closely. The advantage of this approach is that $\tau$ value can be adjusted to balance conservatism with optimism.

\textbf{Overall Objective:} The VerifierQ algorithm minimizes the Bellman error augmented with the CQL term. We adapt the approach of \citet{snell2023offline, kostrikov2021offline}, using Implicit Q-learning to approximate the Q-value for each step with a separate value function, while modifying the objective to incorporate the CQL term.

Like previous works, we use a separate value function $V_{\psi}(s')$ to approximate the Q-value $\max_{a'} Q^*(s', a')$. With our adaptation to the Bellman Update (Equation~\ref{eq:td}), the TD error is given by:
\begin{equation}
   L_{Q}(\theta) = \mathbb{E} \left[ \left( \frac{1}{2} (R(s,a) + \gamma V_{\psi}(s')) - Q_{\theta}(s, a) \right)^2 \right]
   \label{eq:verifierqTD}
\end{equation}

We augment this with our CQL term to achieve a more conservative yet optimistic estimation with Equation~\ref{eq:cqliql}:
\[
L_{CQL}(\psi) = \alpha (\mathbb{E}_{s \sim D, a \sim \mu} \left[ L_{2}^{\tau_1} \left( Q_{\hat{\theta}}(s, a) - V_{\psi}(s) \right)\right] - \mathbb{E}_{s \sim D, a \sim \hat{\pi}_{\beta}} \left[ L_{2}^{\tau_2} \left( Q_{\theta}(s, a) - V_{\psi}(s) \right)\right])
\]

The comprehensive objective function of VerifierQ is thus the sum of the Bellman error and the CQL term:
\begin{equation}
   L(\theta, \psi) = L_{Q}(\theta) + L_{CQL}(\psi)
   \label{eq:verifierqLoss}
\end{equation}

To enhance training stability, we employ a Polyak-averaged version of $Q_{\hat{\theta}}$. The hyperparameter $\alpha$ is set to 1 in our experiments, balancing the influence of the CQL term. The overall objective is shown in the Figure~\ref{fig:verifierq}.

This formulation allows VerifierQ to benefit from the conservative nature of CQL while maintaining an optimistic outlook, crucial for effective Q-value estimation in large action spaces characteristic of language models. The expectile regression provides flexibility to adjust $\tau$ values as preferred. By integrating these components, VerifierQ addresses the challenges of overestimation and large action spaces in utterance-level MDPs, providing a robust framework for multi-step reasoning tasks.

%%%%%%%%%%%%%%Experiments and Results%%%%%%%%%%%%%%
\section{Experiments and Results}
We evaluate VerifierQ on mathematical reasoning tasks from GSM8K and MATH datasets \citep{cobbe2021training, hendrycks2021measuring}. We compare VerifierQ with the state-of-the-art Process Reward Model (PRM) \citet{lightman2023lets}, using the same dataset as \citet{wang2024mathshepherd,snell2024scalingllmtesttimecompute} for a fair comparison. We do not include Object Reward Model (ORM) since \citet{wang2024mathshepherd,snell2024scalingllmtesttimecompute,lightman2023lets} already validated PRM's effectiveness over ORM. Due to computational constraints, we use the TinyLlama-1.1B model \citep{zhang2024tinyllamaopensourcesmalllanguage}. 

\subsection{Experimental Setup}

\textbf{Dataset:} We generate a test time compute set using a generator trained on MetaMath \citep{yu2024metamathbootstrapmathematicalquestions}. The generator is finetuned on MetaMath for 2 epochs with a learning rate of 2e-5, followed by LoRA finetuning for 1 epoch to fit the answer style \citep{hu2021loralowrankadaptationlarge}. For each question in the full GSM8K test set and a 500-question subset of MATH (following \citet{lightman2023lets}), we generate 256 answers. The verifier is trained on the MathShepherd dataset \citet{wang2024mathshepherd}, which uses MCTS-generated data with binary rewards (1 for correct, 0 for incorrect).

\textbf{Model Architecture:} Our model consists of a Q-network and a separate value network to prevent single sample overestimation. We employ soft updates to stabilize training with rate 0.01.

\textbf{Training:} We initialize our model with MetaMath pretraining, then train with PRM on MathShepherd for 1 epoch, followed by VerifierQ training. Here are key hyperparameters. Learning rate: 2e-5 for all training phases. Batch size: 64 (crucial for Q-learning stability). Q-learning parameters: $\gamma = 0.99$, $\alpha = 1$ for the CQL term.
For PRM, we continued training from 1 epoch to 2 epochs. Majority Voting uses the raw output from the generator.

\textbf{Evaluation Metrics:} We evaluate the verifier against PRM and Majority Voting using accuracy as the main metric. Following \citet{snell2024scalingllmtesttimecompute,lightman2023lets}, we use minimum evaluation metrics to evaluate the verifier.

\subsection{Results}
We evaluate VerifierQ against PRM (for epoch and two epoches) and Majority Voting on both GSM8K and MATH datasets using minimum evaluation. For VerifierQ, we use $\tau_1 = 0.3$ for GSM8K and $\tau_1 = 0.5$ for MATH.

\begin{figure}[h]
   \centering
   \includegraphics[width=\linewidth]{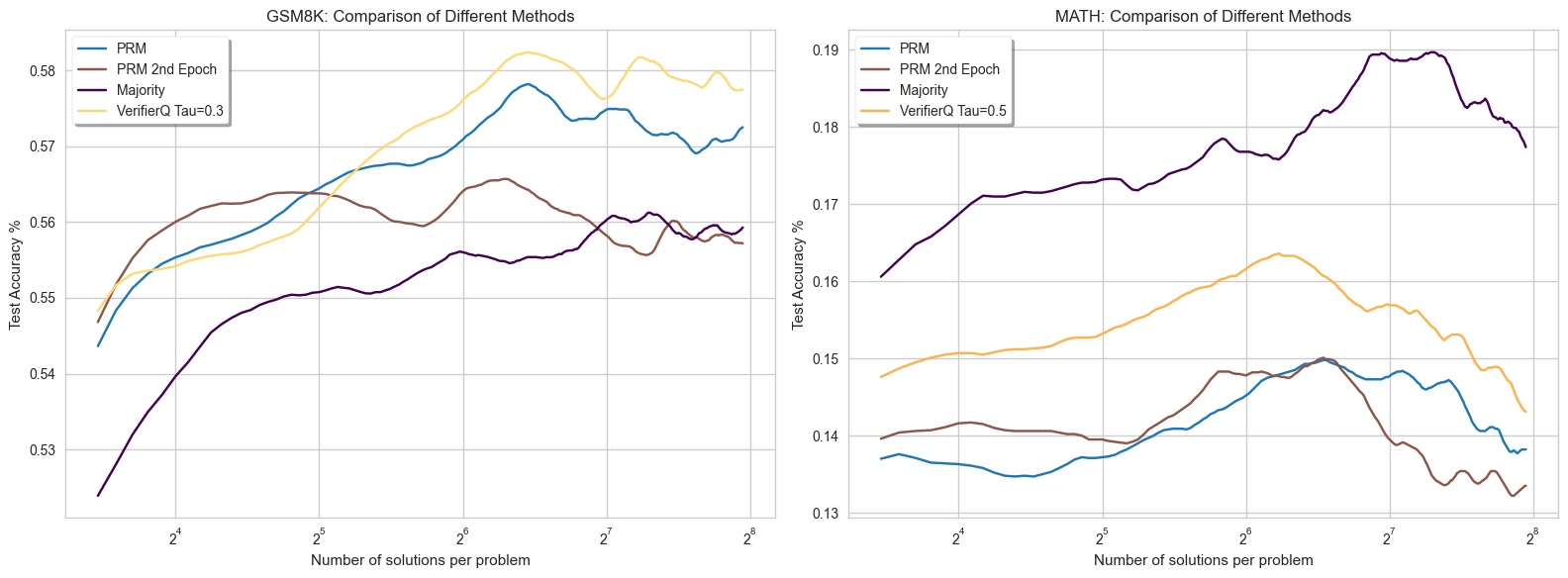}
   \caption{Comparison of different methods on GSM8K (left) and MATH (right) using minimum evaluation. Rolling average over 20 steps. For VerifierQ we use $\tau_1 = 0.3$ (left) and $\tau=0.5$ (right).}
   \label{fig:min-eval}
\end{figure}

As shown in Figure~\ref{fig:min-eval}, VerifierQ outperforms PRM (with 1 epoch), PRM 2nd Epoch and Majority Voting on both datasets. On GSM8K, VerifierQ's performance improves with increase of the number of solutions per problem, aligning with trends observed in previous studies \citep{snell2024scalingllmtesttimecompute,lightman2023lets,wang2024mathshepherd}. We would like to note that for MATH, all methods underperform compared to Majority Vote, possibly due to the small model size (1.1B).

In Figure~\ref{fig:min-eval}, VerifierQ achieves the highest accuracy both on GSM8K and on MATH with different $\tau_1$ values compared to PRM (see Figure~\ref{fig:tau}  in Section~\ref{sec:ablation}), and it also outperforms other Q learning methods (see Figure~\ref{fig:qlearning} in Section~\ref{sec:ablation}). We observe that PRM's performance decreases after the first epoch, likely due to overfitting. Therefore we will use first epoch of PRM hereafter for evaluation and ablation. Different $\tau_1$ values have different performance on two datasets (see Figure~\ref{fig:tau} in Section~\ref{sec:ablation}).

These results demonstrate the potential of applying classic reinforcement learning to verifier models for multi-step reasoning language tasks. They also highlight VerifierQ's effectiveness, particularly on GSM8K, and identify areas for future investigation, such as the impact of model size and the optimization of the values of $\tau$ for different datasets.

%%%%%%%%%%%%%%Ablation Study%%%%%%%%%%%%%%
\section{Ablation Study}
\label{sec:ablation}
Our ablation study addresses the challenges of large action spaces, computational efficiency, and the impact of key components in VerifierQ. To be more specific, we investigate the efficiency of IQL compared to sampling approaches, the effect of the CQL term, the impact of different expectile choices, and the stability of Q-learning.

% In your document where you want the figures:
\begin{figure}[htbp]
    \centering
    \begin{subfigure}[b]{0.48\textwidth}
        \centering
        \includegraphics[width=\textwidth]{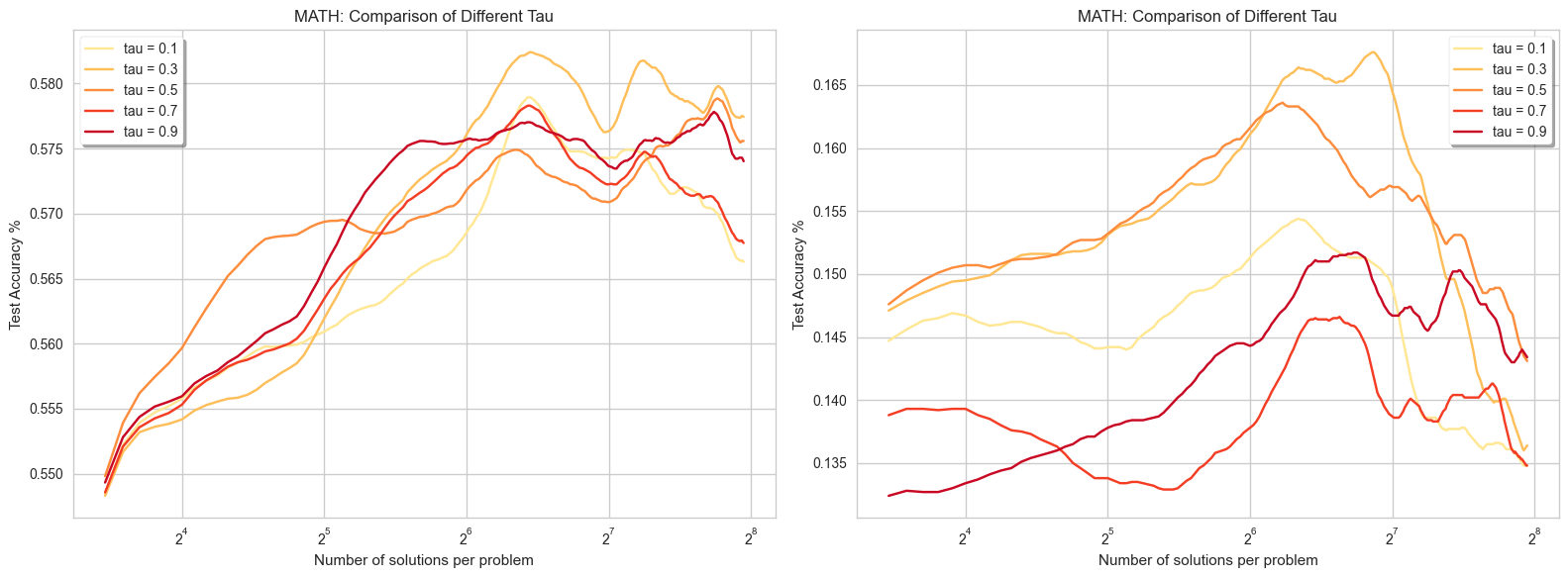}
        \caption{Impact of different $\tau_1$ values on VerifierQ performance. Left: GSM8K dataset. Right: MATH dataset.}
        \label{fig:tau}
    \end{subfigure}
    \hfill
    \begin{subfigure}[b]{0.48\textwidth}
        \centering
        \includegraphics[width=0.95\textwidth]{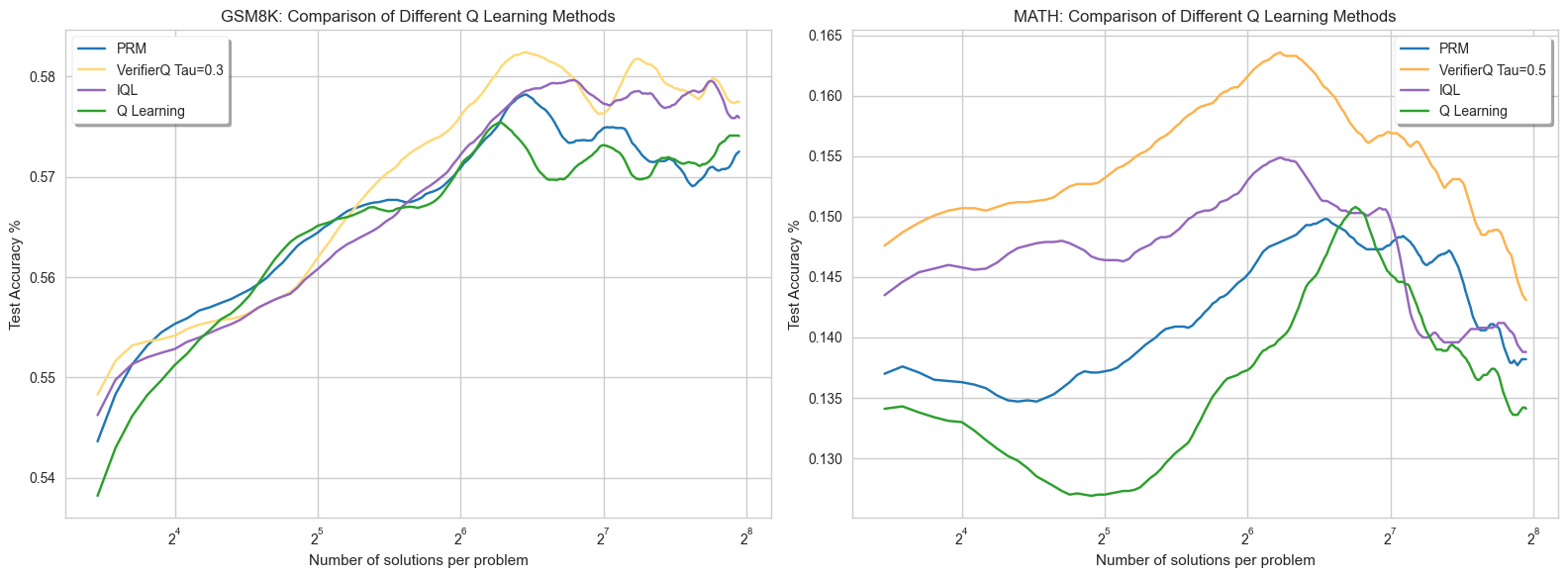}
        \caption{Comparison of Q-Learning Methods. Left: GSM8K dataset. Right: MATH dataset.}
        \label{fig:qlearning}
    \end{subfigure}
    \caption{Comparison of VerifierQ performance and Q-Learning methods}
    \label{fig:combined}
\end{figure}

\subsection{Computational Efficiency}

VerifierQ's sentence-level approach offers significant computational advantages over existing utterance BERT-type and online approaches. We use a SARSA-style approximation of the maximum Q-value, reducing computation time. For online approaches, sampling approaches require $n*m$ action samples for a sentence with $m$ steps to calculate $\max Q$, where $n$ is the number of samples per step. In contrast, VerifierQ performs one forward pass for the entire sentence to estimate all $m$ steps and eliminating the need for $n$ sampling at each step. This approach significantly reduces the computational complexity, especially for longer sequences. From our preliminary experiments, compared to step by step Q value estimation we can save roughly 10x training time.

\subsection{Impact of Adjustable CQL Term}
Figure~\ref{fig:tau} illustrates the impact of different $\tau$ values on VerifierQ's performance.  We examine different levels of optimism by varying $\tau_1$ (0.1, 0.3, 0.5, 0.7, 0.9) while fixing $\tau_2$ at 0.9 to tighten the lower bound to the maximum of the data distribution. As shown in Figure~\ref{fig:tau}, $\tau_1 = 0.3$ generally yields better results, suggesting it approximates the maximum Q-value more effectively than other $\tau_1$ values.
MATH dataset shows higher sensitivity to $\tau_1$ values, with $\tau_1 = 0.5$ performing the best. The difference in optimal $\tau$ values between datasets suggests that dataset-specific tuning may be necessary.

\subsection{Comparison of Q-Learning Methods}

We conduct a comprehensive comparison of VerifierQ against other Q-learning variants to empirically validate the effectiveness of our approach, particularly the impact of the CQL term. We compare VerifierQ with SARSA-style standard Q-learning without CQL and Implicit Q-learning (IQL) with $\tau = 0.9$, also without CQL \citep{snell2023offline}.

Figure~\ref{fig:qlearning} shows VerifierQ outperforming both standard Q-learning and IQL on the GSM8K dataset. VerifierQ's superior performance demonstrates the CQL component's significant contribution and its effectiveness in reducing overestimation. The adjustable $\tau$ terms in VerifierQ allow finer control over the conservatism-optimism balance in Q-value estimation, enabling more optimistic $\max Q$ selection when appropriate. Standard Q-learning performs similarly to PRM, highlighting the substantial improvements achieved by VerifierQ. VerifierQ's adjustable $\tau$ terms ($\tau_1$ and $\tau_2$) allow for finer control over the balance between conservatism and optimism in Q-value estimation. These results empirically validate that VerifierQ's CQL term leads to tangible performance gains, effectively addressing the challenges of applying Q-learning to large language models.

\begin{figure}[!b]
   \centering
   \includegraphics[width=0.8\linewidth]{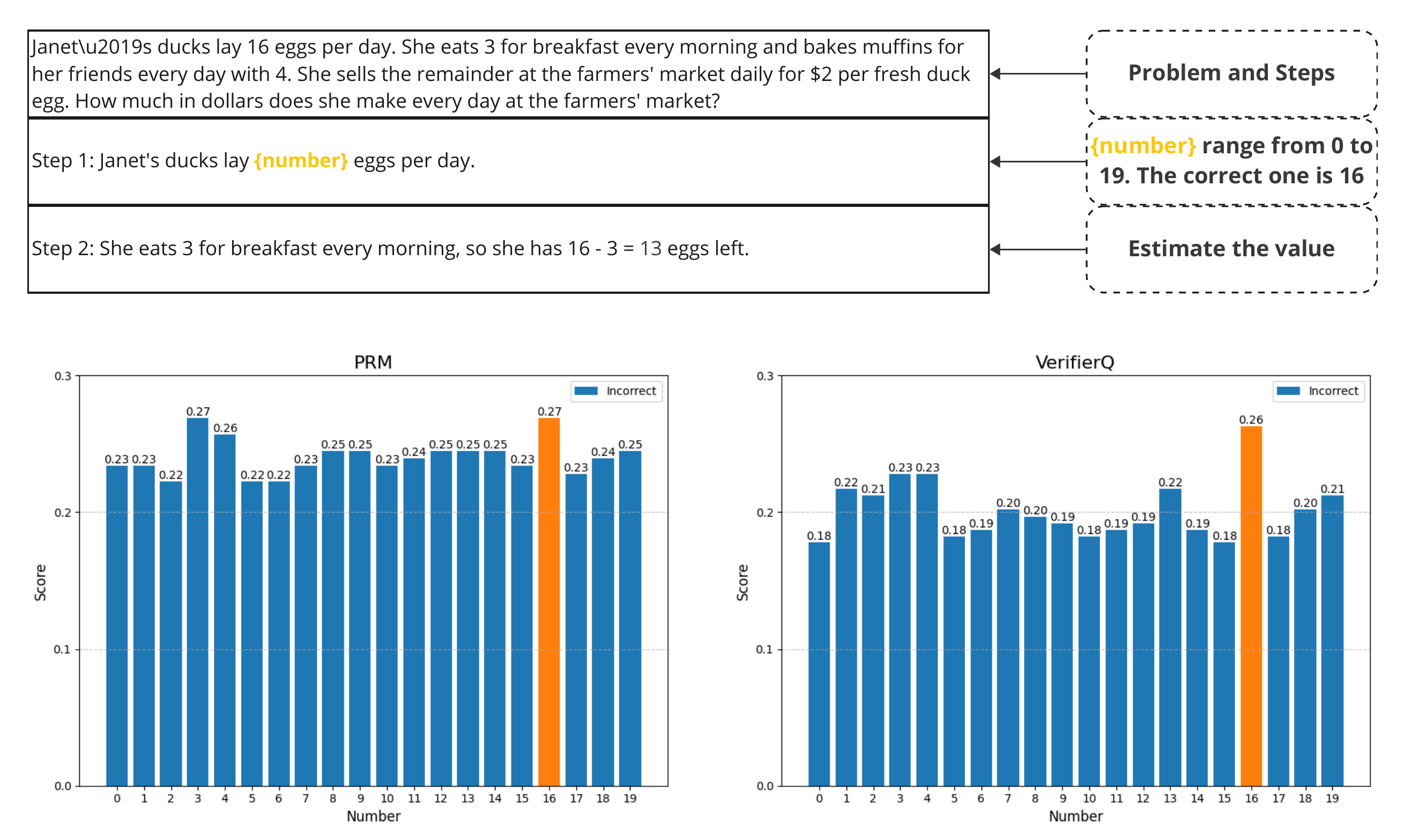}
   \caption{Overestimation case study: PRM (left) vs VerifierQ (right). Orange indicates correct value, blue indicates incorrect value.}
   \label{fig:casestudy}
\end{figure}

\subsection{Overestimation Analysis}
We conduct a qualitative study on overestimation between PRM and VerifierQ by replacing important tokens in the solution sequence with incorrect ones. Figure~\ref{fig:casestudy} reveals that PRM generally assigns higher Q-values to incorrect tokens, while VerifierQ assigns lower values. This suggests that VerifierQ's conservative approach helps mitigate the overestimation issue by assigning lower values to out-of-distribution tokens.

%%%%%%%%%%%%%%Discussion and Future Work%%%%%%%%%%%%%%
\section{Discussion, Ethics, and Limitations}
VerifierQ demonstrates the potential of integrating classic reinforcement learning techniques with language models to enhance multi-step reasoning capabilities. This approach opens up several avenues for future research and applications: The flexibility of our language-based RL framework allows for potential extensions beyond mathematical reasoning. For example, VerifierQ could be applied to guide and verify complex programming tasks in code generation and enhance decision-making in diverse domains requiring multi-step planning. 

In addition, the actor-critic model in language models could lead to more sophisticated planning and decision-making capabilities. Existing success in AI explored the actor-critic model, and actor and critic model in language models could enhance planning and decision-making capabilities.

The development and deployment of VerifierQ also raise important ethical considerations. The alignment of the reward function with human values is crucial. As the model's decision-making process becomes more complex, ensuring transparency and maintaining human understanding for its outputs becomes increasingly challenging but vastly important.

While VerifierQ demonstrates promising results, several limitations should be acknowledged: First, due to computational constraints, our experiments were limited to the TinyLlama model. Testing on larger models could potentially yield different but more likely more robust results. Second, the model's performance is highly sensitive to hyperparameter choices. Resource constraints limited our ability to conduct extensive hyperparameter tuning, which could potentially improve results, but this needs future research. Finally, the fully implemented VerifierQ model is more memory-intensive and computationally expensive than the PRM method. Future research should focus on reducing these requirements to enhance the model's efficiency and scalability.

%%%%%%%%%%%%%%Conclusion%%%%%%%%%%%%%%

\section{Conclusion}

This work introduces VerifierQ, a novel approach integrating classical reinforcement learning techniques with language models to enhance multi-step reasoning capabilities. Our key contributions include: 

\textbf{(1).} A flexible architecture for applying Q-learning to utterance-level MDPs in language models. It can estimate multiple utterances level Q values with large action spaces, and easy to extend Q learning, IQL, and CQL.
\textbf{(2).} An innovative formulation of Conservative Q-learning tailored for large action spaces in language tasks. It helps to reduce the overestimation in offline Q learning.
\textbf{(3).} Empirical evidence demonstrating VerifierQ's effectiveness in mathematical reasoning tasks. These results highlight the potential for extending this approach to larger language models and improving test-time compute.

VerifierQ validates the integration of RL into verifier models and demonstrates its potential to enhance test-time compute results. Moreover, it bridges the gap between classic critic models in RL and verifier models in language tasks. It serves as an addition for applying RL in verifier LLMs and paves the way for actor-critic models to achieve more sophisticated artificial intelligence. As we continue to refine and expand upon this approach, VerifierQ opens up new avenues for developing more capable and robust AI systems across a wide range of complex reasoning tasks.

%Temporarily removed the Acknowledgements section due to reviewing requirements
\subsubsection*{Acknowledgments}
%Use unnumbered third level headings for the acknowledgments. All acknowledgments, including those to funding agencies, go at the end of the paper.
\label{sec:acknowledgements}
The work is supported by the National Science Foundation (NSF) through Awards \#2131186 (CISE-MSI),  \#1827505 (PFI), and \#1737533 (S\&CC), and the US Air Force Office of Scientific Research (AFOSR) via Award \#FA9550-21-1-0082. The work is also supported by a College-wide Research Vision (CRV) Fund from the CCNY Provost's Office, and the ODNI Intelligence Community Center for Academic Excellence (IC CAE) at Rutgers University (\#HHM402-19-1-0003 and \#HHM402-18-1-0007).

\newpage
\section*{Reproducibility Statement}
\label{sec:reproducibility}

To ensure reproducibility of our results, we provide the following details:

\textbf{Implementation Details:} The complete implementation details are available in Appendix~\ref{app:algorithm}. The code implementation will be made publicly available at a future date.

\textbf{Hardware Requirements:}
\begin{itemize}
    \item VerifierQ experiments: Conducted on a single NVIDIA A100 GPU with 40GB memory.
    \item Other models (Q Learning and PRM): Can be trained on an NVIDIA RTX 4090 GPU.
\end{itemize}

\textbf{Training Time:}
\begin{itemize}
    \item VerifierQ: Approximately 10 hours for 1 epoch.
    \item PRM: Approximately 4 hours for 1 epoch.
\end{itemize}

\textbf{Datasets:} We use the following publicly available datasets:
\begin{itemize}
    \item MetaMath: \url{https://huggingface.co/datasets/meta-math/MetaMathQA}
    \item GSM8K: \url{https://huggingface.co/datasets/openai/gsm8k}
    \item MathShepherd: \url{https://huggingface.co/datasets/peiyi9979/Math-Shepherd}
    \item MATH (test subset): We use the same dataset as \citet{lightman2023lets}, available at \url{https://github.com/openai/prm800k}
\end{itemize}

\textbf{Model:} All experiments were conducted using the TinyLlama-1.1B model.

\textbf{Hyperparameters:} Key hyperparameters include:
\begin{itemize}
    \item Learning rate: 2e-5 (constant for all training phases)
    \item Batch size: 64
    \item Discount factor ($\gamma$): 0.99
    \item CQL coefficient ($\alpha$): 1
    \item Soft update coefficient ($\alpha_{\text{soft}}$): 0.01
    \item IQL coefficients: 
        \begin{itemize}
            \item For GSM8K: $\tau_1 = 0.3$, $\tau_2 = 0.9$
            \item For MATH: $\tau_1 = 0.5$, $\tau_2 = 0.9$
        \end{itemize}
\end{itemize}

\textbf{Training Process:} The model is initialized with MetaMath pretraining, followed by 1 epoch of PRM training before VerifierQ training begins.

For any additional details or clarifications needed to reproduce our results, please refer to the code and documentation that will be made available upon acceptance.

\newpage

\bibliography{iclr2025_conference}
\bibliographystyle{iclr2025_conference}

\newpage

\appendix

\section{Appendix}

\subsection{Architecture Details}
\label{app:architecture}
To apply Offline Q-learning to LLMs at the utterance level, we propose a flexible architecture that integrates with language modeling tasks. Following \citet{wang2024mathshepherd} and \citet{lightman2023lets}, we utilize two tokens $+$ and $-$ to represent correct and incorrect states, with a tag token indicating estimation. The probability of the correct token can be interpreted as Q-values ranging from 0 to 1, aligning with the reward structure in the MCTS-generated dataset from \citet{wang2024mathshepherd}. We compute the Q-value for each step as:

\begin{equation}
   Q(s, a) = p(+) = \softmax(\text{logit}_{+}) = \sigmoid(\text{logit}_{+} - \text{logit}_{-})
   \label{eq:qvalue_appendix}
\end{equation}

It is flexible to choose either softmax or sigmoid function to compute the Q-value. We use the sigmoid function in our experiments for more effficiency. The Q-value is computed for each step in the solution sequence, estimating a numerical value in the range of (0, 1). 

This formulation offers several advantages:
\begin{enumerate}
    \item It allows flexible integration for Q-value estimation of utterances of arbitrary length since we can insert the step tag anywhere in the sequence.
    \item It enables parallel estimation of multiple Q-values for multiple steps in a single forward pass, significantly reducing computation time.
    \item This approach seamlessly integrates with existing language modeling tasks.
\end{enumerate}

%Subsection 1.1
\subsection{Convergence of Modified Bellman Update}

We first prove that our modified Bellman update converges to a fixed point.

\begin{theorem}[Convergence of Modified Bellman Update]
\label{thm:modified_bellman_convergence}
Let $Q^*$ be the optimal Q-function. The modified Bellman update
\begin{equation}
    Q^*(s, a) = \frac{1}{2} (R(s,a) + \gamma \max_{a'} Q^*(s', a'))
\end{equation}
converges to a unique fixed point.
\end{theorem}

\begin{proof}
Let $\mathcal{T}$ be the operator defined by our modified Bellman equation:
\begin{equation}
    \mathcal{T}Q(s, a) = \frac{1}{2} (R(s,a) + \gamma \max_{a'} Q(s', a'))
\end{equation}

We need to show that $\mathcal{T}$ is a contraction mapping in the sup-norm $\|\cdot\|_\infty$. For any two Q-functions $Q_1$ and $Q_2$:

\begin{align}
    \|\mathcal{T}Q_1 - \mathcal{T}Q_2\|_\infty &= \sup_{s, a} \left| \mathcal{T}Q_1(s, a) - \mathcal{T}Q_2(s, a) \right| \\
    &= \sup_{s, a} \left| \frac{1}{2} \gamma \left( \max_{a'} Q_1(s', a') - \max_{a'} Q_2(s', a') \right) \right| \\
    &\leq \frac{1}{2} \gamma \sup_{s, a} \left| \max_{a'} Q_1(s', a') - \max_{a'} Q_2(s', a') \right| \\
    &\leq \frac{1}{2} \gamma \sup_{s', a'} \left| Q_1(s', a') - Q_2(s', a') \right| \\
    &= \frac{1}{2} \gamma \|Q_1 - Q_2\|_\infty
\end{align}

Since $0 < \gamma < 1$, it follows that $0 < \frac{1}{2} \gamma < 1$. Therefore, $\mathcal{T}$ is a contraction mapping with contraction factor $L = \frac{1}{2} \gamma$. By the Banach fixed-point theorem, $\mathcal{T}$ has a unique fixed point, and the Q-learning algorithm will converge to this fixed point.
\end{proof}

%Subsection 1.2
\subsection{Optimality of IQL in Large Action Spaces}

Next, we prove that IQL can effectively approximate the maximum and minimum Q-value in large action spaces.

\begin{theorem}[IQL Optimality]
\label{thm:iql_optimality}
We can directly get the following result from the proof in \citep{kostrikov2021offline}. As the quantile level $\tau$ approaches 1, the IQL value function $V_\psi$ converges to the maximum Q-value:
\begin{equation}
    \lim_{\tau \to 1} V_\psi(s) = \max_{a \in \mathcal{A}, \pi_{\beta}(a|s) > 0} Q^*(s, a)
\end{equation}

Additionally, as $\tau \to 0$, the IQL value function $V_\psi$ converges to the minimum Q-value:
\begin{equation}
    \lim_{\tau \rightarrow 0} V_\tau(s) = \min_{a \in \mathcal{A}} Q^*(s, a)
\end{equation}
\end{theorem}

\begin{proof}[Proof Sketch]
Following Lemma 1 of \citet{kostrikov2021offline}, we can show a modified Lemma.
Let X be a real-valued random variable with bounded support and infimum $x^*$. Since X is bounded below and $m_\tau$ approaches the lower bound as $\tau \rightarrow 0$, we have:
\begin{equation}
    \lim_{\tau \rightarrow 0} m_\tau = \inf \{ x : F_X(x) > 0 \} = x^*
\end{equation}

For all $\tau_1$ and $\tau_2$ such that $0 < \tau_1 < \tau_2 < 1$, we can get $m_{\tau_1} \leq m_{\tau_2}$. Therefore, as $\tau \to 0$, the limit of $m_\tau$ converges to the infimum of the random variable X.

In addition, using Lemma 2 of \citet{kostrikov2021offline}, we can show that the IQL value function $V_\psi$ converges to the minimum Q-value as $\tau \to 0$:

\begin{lemma}
\label{lem:monotonicity}
For all s, $\tau_1$ and $\tau_2$ such that $0 < \tau_1 < \tau_2 < 1$, we have $V_{\tau_1}(s) \leq V_{\tau_2}(s)$. 
\end{lemma}

Since $Q^*(s, a)$ is bounded below, the minimum Q-value exists and is finite. Therefore, as $\tau \to 0$, the IQL value function $V_\psi$ converges to the minimum Q-value:
\begin{equation}
    \lim_{\tau \rightarrow 0} V_\tau(s) = \inf_{a \in \text{supp}(\pi_\beta)} Q^*(s, a)
\end{equation}

So we have:
\begin{equation}
    \lim_{\tau \rightarrow 0} V_\tau(s) = \min_{a \in \mathcal{A}, \pi_{\beta}(a|s) > 0} Q^*(s, a)
\end{equation}
\end{proof}

%Subsection 1.3
\subsection{Conservative Yet Optimistic Q-values with Modified CQL}

Finally, we present a proposition about our modified CQL approach and its potential to lead to conservative yet optimistic Q-values.

\begin{proposition}[Modified CQL Bounds]
\label{prop:modified_cql_bounds}
The modified CQL objective with expectile levels $\tau_1$ (close to 0) and $\tau_2$ (close to 1) aims to provide both lower and upper bounds on the true Q-function $Q^*(s, a)$:
\begin{equation}
    \max_{a \sim \hat{\pi}_\beta} Q_\theta(s, a) \lesssim Q^*(s, a) \lesssim \min_{a \sim \mu} Q_{\hat{\theta}}(s, a)
\end{equation}
where $\lesssim$ denotes "approximately less than or equal to".
\end{proposition}

\begin{remark}[Supporting Arguments and Intuitions]

The original CQL objective is:
\begin{equation}
   \arg \min_{Q} \alpha (\mathbb{E}_{s \sim D, a \sim \mu} \left[ Q(s, a) \right] - \mathbb{E}_{s \sim D, a \sim \hat{\pi}_{\beta}} \left[ Q(s,a) \right])
\end{equation}

Where $\mu$ is the target policy distribution and $\hat{\pi}_{\beta}$ is the data distribution. Intuitively, this term finds the maximum Q-value under the target policy distribution $\mathbb{E}_{s \sim D, a \sim \mu} \left[ Q(s, a) \right]$ and minimizes it since it is usually overestimated. To get a tighter bound, it pushes the Q-value up under the data distribution $\mathbb{E}_{s \sim D, a \sim \hat{\pi}_{\beta}} \left[ Q(s,a) \right]$.

For large action spaces, CQL typically uses importance sampling to estimate $\mathbb{E}_{s \sim D, a \sim \mu} [ Q(s, a)]$ with $\log \sum{a} \exp (Q(s, a))$ at every state \citep{kumar2020conservativeqlearningofflinereinforcement}. However, unlike token-level approaches, we leverage IQL to approximate Q-values in the large action space. This mitigates the requirement to sample a set number of actions for each state and allows for more efficient Q-value estimation for longer sequences.

We propose a novel formulation that directly approximates both the lower bound Q-function and the upper bound of the data distribution using IQL with different $\tau$ values for each term in CQL objective. The goal remains the same: finding the overestimated Q-value under the target policy to minimize it and tighten the bound with the data distribution. However we want to give control on the level of the tightening of the bound.

Our modified CQL objective is:
\begin{equation}
    \begin{split}
        L_{CQL}(\psi) = \alpha (&\mathbb{E}_{s \sim D, a \sim \mu} [L_2^{\tau_1}(Q_{\hat{\theta}}(s, a) - V_\psi(s))] \\
        &- \mathbb{E}_{s \sim D, a \sim \hat{\pi}_\beta} [L_2^{\tau_2}(Q_\theta(s, a) - V_\psi(s))])
    \end{split}
\end{equation}

The first term, with $\tau_1$ close to 0, approximates the lower bound of $Q_{\hat{\theta}}$. It acts as an upper bound on the target policy which is typically overestimated. This suggests:
\begin{equation}
    V_\psi(s) \lesssim \min_{a \sim \mu} Q_{\hat{\theta}}(s, a)
\end{equation}

The second term, with $\tau_2$ close to 1, approximates an upper bound on $Q_\theta$. It acts as a lower bound on the data distribution, indicating:
\begin{equation}
    V_\psi(s) \gtrsim \max_{a \sim \hat{\pi}_\beta} Q_\theta(s, a)
\end{equation}

This approach allows for a more optimistic Q-value estimation within the CQL framework. The lower bound of the target policy $\mu$ pushes the Q-value down less aggressively, while the upper bound of the data distribution $\hat{\pi}_{\beta}$ elevates the Q-value more, resulting in a more optimistic Q-value under the CQL term. This approach maintains the benefits of CQL's conservatism while allowing for adaptable optimism through the adjustment of $\tau_1$ and $\tau_2$.

The modified CQL objective aims to minimize the difference between the lower bound of the overestimated Q-values ($Q_{\hat{\theta}}$) and the upper bound of the true Q-values ($Q_\theta$). Minimizing this difference may lead to a more accurate estimation of $Q^*(s, a)$. We can express this as:
\begin{equation}
    L_{CQL}(\psi) \approx \min_{a \sim \mu} Q_{\hat{\theta}}(s, a) - \max_{a \sim \hat{\pi}_\beta} Q_\theta(s, a) 
\end{equation}

As this difference approaches zero, it suggests that the the lower bound of the overestimation of Q-values is being reduced to the extent supported by the data, and we should have $\max_{a \sim \hat{\pi}_\beta} Q_\theta(s, a) \lesssim \min_{a \sim \mu} Q_{\hat{\theta}}(s, a)$. Adjusting $\tau_1$ we could have $Q_{\hat{\theta}}(s, a)$ approximately close to the optimal maximum.

This formulation allows us to balance conservatism with optimism in Q-value estimation. The lower bound of the Q-value is pushed down less aggressively, while the upper bound is elevated more, resulting in Q-values that approach the maximum Q-value under the data distribution more closely. We can adjust $\tau_1$ and $\tau_2$ to fine-tune this balance, allowing for more adaptable Q-values under the CQL term.

It's important to note that this difference can potentially become negative. A negative value would imply that the estimated lower bound of $Q_{\hat{\theta}}$ is smaller than the estimated upper bound of $Q_\theta$ for some state-action pairs. While this might seem counterintuitive given the general overestimation tendency of $Q_{\hat{\theta}}$, it can occur due to the approximations introduced by the $L_2^\tau$ loss functions or other factors in the learning process. This suggests that the value function might be correctly valuing the in-distribution actions more highly, which is desirable, although it might introduce some pessimism in the value estimates.

This intuition provides insight into why our modified CQL approach might lead to a bit more optimistic Q-values. However, a rigorous mathematical proof would require further development and analysis.
\end{remark}

\newpage
\section{Appendix}
\subsection{Algorithm and Implementation Details}
\label{app:algorithm}

\definecolor{color}{rgb}{0, 0.5, 0}
\begin{algorithm}[H]
\caption{VerifierQ}
\label{alg:verifierq}
\begin{algorithmic}
\State \textbf{Input:} Dataset $D$, Q-network $Q_{\theta}$, target Polyak-averaged Q-network $Q_{\hat{\theta}}$ with $\alpha_{\text{soft}}$, value network $V_{\psi}$, IQL coefficients $\tau_1$ and $\tau_2$, CQL coefficient $\alpha$
\State Initialize Q-network $Q_{\theta}$, target Q-network $Q_{\hat{\theta}}$, value network $V_{\psi}$
\State Initialize target Q-network parameters $\hat{\theta} \leftarrow \theta$
\For{each training step}
   \State Sample batch of state-action pairs $S = (s_1, s_2, s_3) \sim D$ and rewards $R = (r_1, r_2, r_3) \sim D$
   \State \textcolor{color}{\# TD Target.}
   \State Compute target Q-values in parallel: $y = \frac{1}{2} (r + \gamma V_{\psi}(S'))$ \Comment{Equation~\ref{eq:bellman}}
   \State TD Loss: $L_{Q}(\theta) = \frac{1}{2} (Q_{\theta}(S) - y)^2$ \Comment{Equation~\ref{eq:verifierqTD}}
   \State \textcolor{color}{\# CQL Term.}
   \State Compute CQL $\mu$ with IQL: $L_{\mu} = L_2^{\tau_1}(Q_{\hat{\theta}}(S) - V_{\psi}(S))$ \Comment{Equation~\ref{eq:iql}}
   \State Compute CQL  $\hat{\pi}_{\beta}$ with IQL: $L_{\hat{\pi}} = L_2^{\tau_2}(Q_{\theta}(S) - V_{\psi}(S))$ \Comment{Equation~\ref{eq:iql}} 
   \State CQL Loss: $L_{CQL}(\psi) = \alpha (L_{\mu} - L_{\hat{\pi}})$ \Comment{Equation~\ref{eq:cqliql}} 
   \State \textcolor{color}{\# Update networks}
   \State Update Q-network: $\theta \leftarrow \theta - \nabla_{\theta} L_{Q}(\theta)$ 
   \State Update value network: $\psi \leftarrow \psi - \nabla_{\psi} L_{CQL}(\psi)$ 
   \State Update target Q-network: $\hat{\theta} \leftarrow (1 - \alpha_{\text{soft}} )\hat{\theta} + \alpha_{\text{soft}} \theta$
\EndFor
\end{algorithmic}
\end{algorithm}

As described in Section~\ref{sec:background}, the state at step $i$ is the concatenation of the problem statement and all tokens generated up to that point: $s_i = [p, a_1, a_2, \dots, a_i]$. As illustrated in Figure~\ref{fig:RLLLM}, $s_1$ consists of $p$ and $a_1$, $s_2$ consists of $p$, $a_1$, and $a_2$, and so on. The reward $r_i$ is 1 if the token $a_i$ is correct and 0 otherwise. This approach leverages the decoder architecture's ability to generate the next token based on the previous tokens.

For the hyperparameters, we use the following settings:
\begin{itemize}
    \item Discount factor: $\gamma = 0.99$
    \item CQL coefficient: $\alpha = 1$
    \item Soft update coefficient: $\alpha_{\text{soft}} = 0.01$
    \item Batch size: 64
    \item Optimizer: AdamW with a constant learning rate of $2e-5$ for all training phases
    \item IQL coefficients: 
        \begin{itemize}
            \item For GSM8K: $\tau_1 = 0.3$, $\tau_2 = 0.9$
            \item For MATH: $\tau_1 = 0.5$, $\tau_2 = 0.9$
        \end{itemize}
\end{itemize}

We initialize the model with MetaMath pretraining and train it with PRM for 1 epoch before starting VerifierQ training. All experiments are conducted using the TinyLlama-1.1B model.

\end{document}